\documentclass{article}

\usepackage{microtype}
\usepackage{graphicx}
\usepackage{subfig}
\usepackage{booktabs} % for professional tables

\usepackage[accepted]{icml2020}
\usepackage{amsmath,amsfonts,amsthm}
\usepackage{pgf,tikz,pgfplots}
\usepackage{subfig}
\pgfplotsset{compat=1.14}
\usepackage{mathrsfs}
\usetikzlibrary{arrows}
\usepackage{color}

\icmltitlerunning{When are Non-parametric Methods Robust?}
\newtheorem{thm}{Theorem}
\newtheorem{lem}[thm]{Lemma}
\newtheorem{cor}[thm]{Corollary}

\newtheorem{defn}[thm]{Definition}

\def\D{{\mathcal D}}
\def\X{\mathcal X}
\def\R{\mathbb R}
\def\Y{\{\pm 1\}}

\def\P{\mathbb{P}}

\def\b{g^*}

\def\rcons{r-consistent}
\def\rconsy{r-consistency}
\def\ap{AdvPrun}
\def\ga{RobustNonPar}

\DeclareMathOperator*{\argmax}{arg\,max}

\begin{document}

\twocolumn[
\icmltitle{When are Non-Parametric Methods Robust?}



\icmlsetsymbol{equal}{*}

\begin{icmlauthorlist}
\icmlauthor{Robi Bhattacharjee}{equal,to}
\icmlauthor{Kamalika Chaudhuri}{equal,to}
\end{icmlauthorlist}

\icmlaffiliation{to}{Department of Computer Science, University of California, San Diego}
\icmlcorrespondingauthor{Robi}{rcbhatta@eng.ucsd.edu}

\icmlkeywords{Machine Learning, ICML}

\vskip 0.3in
]

\printAffiliationsAndNotice{\icmlEqualContribution} 

\begin{abstract}
A growing body of research has shown that many classifiers are susceptible to {\em{adversarial examples}} -- small strategic modifications to test inputs that lead to misclassification. In this work, we study general non-parametric methods, with a view towards understanding when they are robust to these modifications. We establish general conditions under which non-parametric methods are $r$-consistent -- in the sense that they converge to optimally robust and accurate classifiers in the large sample limit. 

Concretely, our results show that when data is well-separated, nearest neighbors and kernel classifiers are $r$-consistent, while histograms are not. For general data distributions, we prove that preprocessing by Adversarial Pruning~\cite{YRWC19} -- that makes data well-separated -- followed by nearest neighbors or kernel classifiers also leads to $r$-consistency. 
\end{abstract}

\section{Introduction}

Recent work has shown that many classifiers tend to be highly non-robust and that small strategic modifications to regular test inputs can cause them to misclassify~\cite{Szegedy14, Goodfellow14, MeekLowd05}. Motivated by the use of machine learning in safety-critical applications, this phenomenon has recently received considerable interest; however, what exactly causes this phenomenon -- known in the literature as {\em{adversarial examples}} -- still remains a mystery.

Prior work has looked at three plausible reasons why adversarial examples might exist. The first, of course, is the possibility that in real data distributions, different classes are very close together in space -- which does not seem plausible in practice. Another possibility is that classification algorithms may require more data to be robust than to be merely accurate; some prior work~\cite{Madry18, WJC18, Srebro19} suggests that this might be true for certain classifiers or algorithms. Finally, others~\cite{Bubeck19, Vinod19, WJC18} have suggested that better training algorithms may give rise to more robust classifiers -- and that in some cases, finding robust classifiers may even be computationally challenging.

In this work, we consider this problem in the context of general non-parametric classifiers. Contrary to parametrics, non-parametric methods are a form of local classifiers, and include a large number of pattern recognition methods such as nearest neighbors, decision trees, random forests and kernel classifiers. There is a richly developed statistical theory of non-parametric methods~\cite{devroye96}, which focuses on accuracy, and provides very general conditions under which these methods converge to the Bayes optimal with growing number of samples. We, in contrast, analyze robustness properties of these methods, and ask instead when they converge to the classifier with the highest astuteness at a desired radius $r$. Recall that the astuteness of a classifier at radius $r$ is the fraction of points from the distribution on which it is accurate and has the same prediction up to a distance $r$~\cite{WJC18, Madry18}.

 We begin by looking at the very simple case when data from different classes is well-separated -- by at least a distance $2r$. Although achieving astuteness in this case may appear trivial, we show that even in this highly favorable case, not all non-parametric methods provide robust classifiers -- and this even holds for methods that converge to the Bayes optimal in the large sample limit.  

This raises the natural question -- when do non-parametric methods produce astute classifiers? We next provide conditions under which a non-parametric method converges to the most astute classifier in the large sample limit under well-separated data. Our conditions are analogous to the classical conditions for convergence to the Bayes optimal~\cite{devroye96, Stone77}, but a little stronger. We show that nearest neighbors and kernel classifiers whose kernel functions decay fast enough, satisfy these conditions, and hence converge to astute classifiers in the large sample limit. In constrast, histogram classifiers, which do converge to the Bayes optimal in the large sample limit, may not converge to the most astute classifier. This indicates that there may be some non-parametric methods, such as nearest neighbors and kernel classifiers, that are more naturally robust when trained on well-separated data, and some that are not.

What happens when different classes in the data are not as well-separated? For this case, \cite{YRWC19} proposes a method called Adversarial Pruning that preprocesses the training data by retaining the maximal set of points such that different classes are distance $\geq 2r$ apart, and then trains a non-parametric method on the pruned data. We next prove that if a non-parametric method has certain properties, then the classifier produced by Adversarial Pruning followed by the method does converges to the most astute classifier in the large sample limit. We show that again nearest neighbors and kernel classifiers whose kernel functions decay faster than inverse polynomials satisfy these properties. Our results thus complement and build upon the empirical results of~\cite{YRWC19} by providing a performance guarantee. 

What can we conclude about the cause for adversarial examples? Our results seem to indicate that at least for non-parametrics, it is mostly the training algorithms that are responsible. With a few exceptions, decades of prior work in machine learning and pattern recognition has largely focussed on designing training methods that provide increasingly accurate classifiers -- perhaps to the detriment of other aspects such as robustness. In this context, our results serve to (a) provide a set of guidelines that can be used for designing non-parametric methods that are robust and accurate on well-separated data and (b) demonstrate that when data is not well-separated, preprocessing through adversarial pruning~\cite{YRWC19} may be used to ensure convergence to optimally astute solutions in the large sample limit.

%Thus, for machine learning methods to be more applicable in practice, there needs to be a similar effort to design training algorithms that aim to produce robust classifiers -- in addition to highly accurate ones. 

\subsection{Related Work}

There is a large body of work on adversarial attacks~\cite{Carlini17, Liu17, Papernot17, Papernot16,Szegedy14} and defenses~\cite{Hein17,Katz17,Schmidt18,Wu16,Steinhardt18, Sinha18} in the parametric setting, specifically focusing on neural networks. On the other hand, adversarial examples for nonparametric classifiers have mostly been studied in a much more ad-hoc manner, and to our knowledge, there has been no theoretical investigation into general properties of algorithms that promote robustness in non-parametric classifiers.

For nearest neighbors, there has been some prior work on adversarial attacks~\cite{Amsaleg17, Sitawarin19, WJC18, YRWC19} as well as defenses. Wang et. al. \cite{WJC18} proposes a defense for 1-NN by pruning the input sample. However, their defense learns a classifier whose robustness regions converge towards those of the Bayes optimal classifier, which itself may potentially have poor robustness properties. Yang et. al. \cite{YRWC19} accounts for this problem by proposing the notion of the $r$-optimal classifier, and propose an algorithm called Adversarial Pruning which can be interpreted as a finite sample approximation to the $r$-optimal. However, they do not provide formal performance guarantees for Adversarial Pruning, which we do. 

For Kernel methods, Hein and Andriushchenko \cite{Hein17} study lower bounds on the norm of the adversarial manipulation that is required for changing a classifiers output. They specifically study bounds for Kernel Classifiers, and propose an empirically based regularization idea that improves robustness. In this work, we improve the robustness properties of kernel classification through adversarial pruning, and show formal guarantees regarding convergence towards the $r$-optimal classifier. 

For decision trees and random forests, attacks and defenses have been provided by \cite{Hein19, Kantchelian15, Hsiehicml19}. Again, most of the work here is empirical in nature, and convergence guarantees are not provided. 

Pruning has a long history of being applied for improving nearest neighbors \cite{Gates72, Gottlieb14, Hart68, KontorovichSW17, KontorovichW15, Hanneke19}, but this has been entirely done in the context of generalization, without accounting for robustness. In their work, Yang et. al. empirically show that adversarial pruning can improve robustness for nearest neighbor classifiers. However, they do not provide any formal guarantees for their algorithms. In this work, we prove formal guarantees for \textit{adversarial pruning} in the large sample limit, both for nearest neighbors as well as for more general \textit{weight functions.} 

There is a long history of literature for understanding the consistency of Kernel classifiers \cite{Steinwart05, Stone77}, but this has only been done for accuracy and generalization. In this work, we find different conditions are needed to ensure that a Kernel classifier converges in robustness in addition to accuracy.

\section{Preliminaries}

\subsection{Setting}
We consider binary classification where instances are drawn from a totally bounded metric space $\X$ that is equipped with distance metric denoted by $d$, and the label space is $\Y = \{ -1, +1 \}$. The classical goal of classification is to build a highly \textit{accurate} classifier, which we define as follows.

\begin{defn}
(Accuracy) Let $\D$ be a distribution over $\X \times \Y$, and let $f \in \Y^\X$ be a classifier. Then the \textbf{accuracy} of $f$ over $\D$, denoted $A(f, \D)$, is the fraction of examples $(x,y) \sim \D$ for which $f(x) = y$. Thus $$A(f, \D) = P_{(x,y) \sim \D}[f(x) = y].$$
\end{defn}

In this work, we consider \textit{robustness} in addition to accuracy. Let $B(x,r)$ denoted the closed ball of radius $r$ centered at $x$. 

\begin{defn}
(Robustness) A classifier $f \in \Y^\X$ is said to be \textbf{robust} at $x$ with radius $r$ if $f(x) = f(x')$ for all $x' \in B(x,r)$.
\end{defn}

Our goal is to find non-parametric algorithms that output classifiers that are robust, in addition to being accurate. To account for both criteria, we combine them into a notion of \textit{astuteness}~\cite{WJC18, Madry18}. 

\begin{defn}
(Astuteness) A classifier $f \in \Y^\X$ is said to be \textbf{astute} at $(x,y)$ with radius $r$ if $f$ is robust at $x$ with radius $r$ and $f(x) = y$. The \textbf{astuteness} of $f$ over $\D$, denoted $A_r(f, \D)$, is the fraction of examples $(x,y) \sim \D$ for which $f$ is astute at $(x,y)$ with radius $r$. Thus $$A_r(f, \D) = P_{(x, y) \sim \D}[f(x') = y, \forall x' \in B(x,r)].$$
\end{defn}

It is worth noting that $A_0(f, \D) = A(f, \D)$, since astuteness with radius $0$ is simply the accuracy. For this reason, we will use $A_0(f, \D)$ to denote accuracy from this point forwards.

\subsection{Notions of Consistency}

Traditionally, a classification algorithm is said to be consistent if as the sample size grows to infinity, the accuracy of the classifier it learns converges towards the best possible accuracy on the underlying data distribution. We next introduce and formalize an alternative form of consistency, called $r$-consistency, that applies to robust classifiers.

We begin with a formal definition of the Bayes Optimal Classifier -- the most accurate classifier on a distribution -- and consistency. 

\begin{defn}
(Bayes Optimal Classifier) The \textbf{Bayes Optimal Classifier} on a distribution $\D$, denoted by $\b$, is defined as follows. Let $\eta(x) = p_\D(+1|x)$. Then
 \[ \b(x) = \begin{cases} 
      +1 & \eta(x) \geq 0.5 \\
      -1 & \eta(x) < 0.5 \\
   \end{cases}
\]
It can be shown that $\b$ achieves the highest accuracy over $\D$ over all classifiers.
\end{defn}

\begin{defn}
(Consistency) Let $M$ be a classification algorithm  over $\X \times \Y$. $M$ is said to be \textbf{consistent} if for any $\D$ over $\X \times \Y$, and any $\epsilon, \delta$ over $(0,1)$, there exists $N$ such that for $n \geq N$, with probability $1-\delta$ over $S \sim \D^n$, we have: $$A(M(S), \D) \geq A(\b, \D) - \epsilon,$$ where $\b$ is the Bayes optimal classifier for $\D$. 
\end{defn}

How can we incorporate robustness in addition to accuracy in this notion? A plausible way, as used in~\cite{WJC18}, is that the classifier should converge towards being astute where the Bayes Optimal classifier is astute. However, the Bayes Optimal classifier is not necessarily the most astute classifier and may even have poor astuteness. To see this, consider the following example. 

\paragraph{Example 1}
Consider $\D$ over $\X = [0,1]$ such that $\D_\X$ is the uniform distribution and $$p(y=1|x) = \frac{1}{2} + \sin \frac{4 \pi x}{r}.$$ For any point $x$, there exists $x_1, x_2 \in ([x-r, x+r] \cap [0,1])$ such that $p(y=1|x_1) > \frac{1}{2}$ and $p(y=1|x_2) < \frac{1}{2}$. $A_r(\b, r) = 0$. However, the classifier that always predicts $f(x) = +1$ does better. It is robust everywhere, and since $P_{(x,y) \sim \D}[y = +1] = \frac{1}{2}$, it follows that $A_r(f, \D) = \frac{1}{2}$. \\ \\

This motivates the notion of the $r$-optimal classifier, introduced by~\cite{YRWC19}, which is the classifier with maximum astuteness. 

\begin{defn}
($r$-optimal classifier) The \textbf{$r$-optimal classifier} of a distribution $G$ denoted by $\b_r$ is the classifier with maximum astuteness. Thus $$\b_r = \argmax_{f \in \Y^\X} A_r(f, \D).$$ We let $A_r^*(\D)$ denote $A_r(\b_r, \D)$. 
\end{defn}

Observe that $\b_r$ is not necessarily unique. To account for this, we use $A_r^*(\D)$ in our definition for \rconsy. 

\begin{defn} \label{defn_archons}
(\rcons) Let $M$ be a classification algorithm over $\X \times \Y$. $M$ is said to be \textbf{\rcons} if for any $\D$,  any $\epsilon, \delta \in (0,1)$, and $0 < \gamma < r$, there exists $N$ such that for $n \geq N$, with probability $1-\delta$ over $S \sim \D^n$, $$A_{r-\gamma}(M(S), \D) \geq A_r^*(\D) - \epsilon.$$ if the above conditions hold for a specific distribution $\D$, we say that $M$ is \rcons\emph{ }with respect to $\D$. 
\end{defn}

Observe that in addition to the usual $\epsilon$ and $\delta$, there is an extra parameter $\gamma$ which measures the gap in the robustness radius. We may need this parameter as when classes are exactly $2r$ apart, we may not be able to find the exact robust boundary with only finite samples. 

Our analysis will be centered around understanding what kinds of algorithms $M$ provide highly astute classifiers for a given radius $r$. We begin by first considering the special case of \textit{$r$-separated} distributions. 

\begin{defn}
($r$-separated distributions) A distribution $\D$ is said to be \textbf{$r$-separated} if there exist subsets $T^+, T^- \subset \X$ such that 
\begin{enumerate}
	\item $\P_{(x,y) \sim \D}[x \in T^y] = 1$. 
	\item $\forall x_1 \in T^+, \forall x_2 \in T^-$, $d(x_1, x_2) > 2r$.
\end{enumerate}
\end{defn}

Observe that if $\D$ is $r$-separated, $A_r(\b_r, \D) = 1$.

\subsection{Non-parametric Classifiers}\label{classifiers}

Many non-parametric algorithms classify points by averaging labels over a local neighborhood from their training data. A very general form of this idea is encapsulated in \textit{weight functions} -- which is the general form we will use.

\begin{defn} \label{def:weight}
\cite{devroye96} A \textbf{weight function} $W$ is a non-parametric classifier with the following properties.
\begin{enumerate}
	\item Given input $S = \{(x_1, y_1), (x_2, y_2,), \dots, (x_n, y_n)\} \sim \D^n$, $W$ constructs functions $w_1^S, w_2^S, \dots, w_n^S: \X \to [0, 1]$ such that for all $x \in \X$, $\sum_1^n w_i^S(x) = 1$. The functions $w_i^S$ are allowed to depend on $x_1, x_2, \dots x_n$ but must be independent of $y_1, y_2, \dots, y_n$. 
	\item $W$ has output $W_S$ defined as \[ W_S(x) = \begin{cases} 
      +1 & \sum_1^n w_i^S(x)y_i > 0 \\
      -1 & \sum_1^n w_i^S(x)y_i \leq 0 \\
   \end{cases}
\]
As a result, $w_i^S(x)$ can be thought of as the weight that $(x_i, y_i)$ has in classifying $x$.
\end{enumerate}
\end{defn}

Weight functions encompass a fairly extensive set of common non-parametric classifiers, which is the motivation for considering them. We now define several common non-parametric algorithms that can be construed as weight functions. 

\begin{defn}
A \textbf{histogram classifier}, $H$, is a non-parametric classification algorithm over $\R^d \times \Y$ that works as follows. For a distribution $\D$ over $\R \times \Y$, $H$ takes $S = \{(x_i, y_i): 1 \leq i \leq n\} \sim \D^n$ as input. Let $k_i$ be a sequence with $\lim_{i \to \infty} k_i = \infty$ and $\lim_{i \to \infty} \frac{k_i}{i} = 0$. $H$ constructs a set of hypercubes $C = \{c_1, c_2, \dots, c_m\}$ as follows:
\begin{enumerate}
	\item Initially $C = \{c\}$, where $S \subset c$.
	\item For $c \in C$, if $c$ contains more than $k_n$ points of $S$, then partition $c$ into $2^d$ equally sized hypercubes, and insert them into $C$.
	\item Repeat step $2$ until all cubes in $C$ have at most $k_n$ points. 
\end{enumerate}
For $x \in \R$ let $c(x)$ denote the unique cell in $C$ containing $x$. If $c(x)$ doesn't exist, then $H_S(x) = -1$ by default. Otherwise, \[ H_S(x) = \begin{cases} 
      +1 & \sum_{x_i \in c(x)} y_i > 0 \\
      -1 & \sum_{x_i \in c(x)}y_i \leq 0 \\
   \end{cases}.
\]
\end{defn}

Histogram classifiers are weight functions in which all $x_i$ contained within the same cell as $x$ are given the same weight $w_i^S(x)$ in predicting $x$, while all other $x_i$ are given weight $0$. 

\begin{defn}
A \textbf{kernel classifier} is a weight function $W$ over $\X \times \Y$ constructed from function $K: \R^+ \cup \{0\} \to \R^+$ and some sequence $\{h_n\} \subset \R^+$ in the following manner. Given $S = \{(x_i, y_i)\} \sim \D^n$, we have $$w_i^S(x) = \frac{K(\frac{d(x, x_i)}{h_n})}{\sum_{j = 1}^n K(\frac{d(x, x_j)}{h_n})}.$$ Then, as above, $W$ has output \[ W_S(x) = \begin{cases} 
      +1 & \sum_1^n w_i^S(x)y_i > 0 \\
      -1 & \sum_1^n w_i^S(x)y_i \leq 0 \\
   \end{cases}
\]
\end{defn}

Finally, we note that $k_n$-nearest neighbors is also a weight function; $w_i^S(x) = \frac{1}{k_n}$ if $x_i$ is one of the $k_n$ closest neighbors of $x$ and $0$ otherwise. 

\section{Warm Up: $r$-separated distributions}

We begin by considering the case when the data distribution is $r$-separated; the more general case is considered in Section~\ref{sec:general}. While classifying $r$-separated distributions robustly may appear almost trivial, learning an arbitrary classifier does not necessarily produce an astute result. To see this, consider the following example of a histogram classifier -- which is known to be consistent.

We let $H$ denote the histogram classifier over $\R$.

\paragraph{Example 2}
Consider the data distribution $\D = \D^+ \cup \D^-$ where $D^+$ is the uniform distribution over $[0, \frac{1}{4})$ and $D^-$ is the uniform distribution over $(\frac{1}{2}, 1]$, $p(+1|x) = 1$ for $x \in \D^+$, and $p(-1|x) = 1$ for $x \in \D^-$. 

We make the following observations (refer to Figure \ref{fig:histogram}).
\begin{enumerate}
	\item $\D$ is $0.1$-separated, since the supports of $\D^+$ and $\D^-$ have distance $0.25 > 0.2$. 
	\item If $n$ is sufficiently large, $H$ will construct the cell $[0.25, 0.5)$, which will not be split because it will never contain any points. 
	\item $H_S(x) = -1$ for $x \in [0.25, 0.5)$.
	\item $H_S$ is not astute at $(x,1)$ for $x \in (0.15, 0.25)$. Thus $A_{0.1}(H_S, \D) = 0.8$.
\end{enumerate}

\begin{figure}
\centering
\definecolor{dtsfsf}{rgb}{0.8274509803921568,0.1843137254901961,0.1843137254901961}
\definecolor{sexdts}{rgb}{0.1803921568627451,0.49019607843137253,0.19607843137254902}
\begin{tikzpicture}[line cap=round,line join=round,>=triangle 45,x=1cm,y=1cm]
\clip(-4.68,0) rectangle (2.46,3);
\draw [line width=2pt] (-3.8971542001619626,1.7601524817792173)-- (-3.90319974718047,1.0800284421971598);
\draw [line width=2pt] (-2.7721542001619626,1.7501524817792171)-- (-2.7781997471804694,1.0700284421971598);
\draw [line width=2pt] (-1.6471542001619623,1.7401524817792173)-- (-1.6531997471804694,1.0600284421971597);
\draw [line width=2pt] (0.6028457998380374,1.720152481779217)-- (0.5968002528195302,1.0400284421971597);
\draw [line width=2pt] (-2.775,1.43)-- (-1.65,1.42);
\draw [line width=2pt,color=sexdts] (-1.65,1.42)-- (0.6,1.4);
\draw [line width=2pt,color=sexdts] (-3.9,1.44)-- (-3.2002330680045032,1.4337798494933733);
\draw [line width=2pt,color=dtsfsf] (-3.2002330680045032,1.4337798494933733)-- (-2.775,1.43);
\draw (-3.54,2.48) node[anchor=north west] {$\mathcal{D}^+$};
\draw (-0.76,2.46) node[anchor=north west] {$\mathcal{D}^-$};
\draw (-4,1.1) node[anchor=north west] {0};
\draw (-3.14,1.08) node[anchor=north west] {0.25};
\draw (-1.86,1.1) node[anchor=north west] {0.5};
\draw (0.5,1.08) node[anchor=north west] {1};
\end{tikzpicture}
\caption{$H_S$ is astute in the green region, but not robust in the red region.} \label{fig:histogram}
\end{figure}
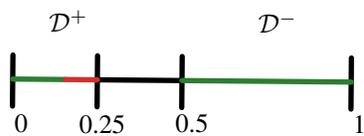

Example 2 shows that histogram classifiers do not always learn astute classifiers even when run on $r$-separated distributions. This motivates the question: which non-parametric classifiers do?

We answer this question in the following theorem, which gives sufficient conditions for a weight function (definition \ref{def:weight}) to be $r$-consistent over an $r$-separated distribution.

\begin{thm}\label{thm_stone_cons}
Let $\D$ be a distribution over $\X \times \Y$, and let $W$ be a weight function. Let $X$ be a random variable with distribution $\D_\X$, and $S = \{(x_1, y_1), (x_2, y_2), \dots, (x_n, y_n)\} \sim \D^n$. Suppose that for any $0 < a < b,$ $$\lim_{n \to \infty} \mathbb{E}_{X, S} \big [ \sup_{x' \in B(X, a)} \sum_1^n w_i^S(x')I_{||x_i - x'|| > b} \big] = 0.$$  Then if $\D$ is $r$-separated, $W$ is \rcons\emph{ } with respect to $\D$.  
\end{thm}

First, we compare Theorem \ref{thm_stone_cons} to Stone's theorem \cite{Stone77}, which gives sufficient conditions for a weight function to be consistent (i.e. converge in accuracy towards the Bayes optimal). For convenience, we include a statement of Stone's theorem. 
\begin{thm}\label{thm_stone}
\cite{Stone77} Let $W$ be weight function over $\X \times \Y$. Suppose the following conditions hold for any distribution $\D$ over $\X \times \Y$.  Let $X$ be a random variable with distribution $\D_\X$, and $S = \{(x_1, y_1), (x_2, y_2), \dots, (x_n, y_n)\} \sim \D^n$. All expectations are taken over $X$ and $S$. 
\begin{enumerate}
	\item There is a constant $c$ such that, for every nonnegative measurable function $f$ satisfying $\mathbb{E} [f(X)] < \infty$, $$\mathbb{E} [\sum_1^n w_i^S(X)f(x_i)] \leq c \mathbb{E} [f(x)].$$
	\item For all $a > 0$, $$\lim_{n \to \infty} \mathbb{E}[\sum_1^n w_i^S(x)I_{||x_i - X|| > a}] = 0,$$ where $I_{||x_i - X|| > a}$ is an indicator variable. 
	\item $$\lim_{n \to \infty} \mathbb{E}[\max_{1 \leq i \leq n} w_i^S(X)] = 0.$$
\end{enumerate}
Then $W$ is consistent. 
\end{thm}
There are two main differences between Theorem \ref{thm_stone_cons} and Stone's theorem.
 \begin{enumerate}
 	\item Conditions 1. and 3. of Stone's theorem are no longer necessary. This is because $r$-separated distributions are well-separated and thus have simpler conditions for consistency. In fact, a slight modification of the arguments of~\cite{Stone77} shows that for $r$-separated distributions, condition 2. alone is sufficient for consistency.
 	\item Condition 2. is strengthened. Instead of requiring the weight of $x_i$'s outside of a given radius to go to $0$ for $X \sim \D$, we require the same to \textit{uniformly} hold over a ball centered at $X$. 
\end{enumerate}

Theorem \ref{thm_stone_cons} provides a general condition that allows us to verify the $r$-consistency of non-parametric methods. We now show below that two common non-parametric algorithms -- $k_n$-nearest neighbors and kernel classifiers with rapidly decaying kernel functions -- satisfy the conditions of Theorem~\ref{thm_stone_cons}.
 
\begin{cor}\label{nn_sep_thm}
Let $\D$ be any $r$-separated distribution. Let $k_n$ be any sequence such that $\lim_{n \to \infty} \frac{k_n}{n} = 0$, and let $M$ be the $k_n$-nearest neighbors classifier on a sample $S \sim \D^n$. Then $M$ is \rcons\emph{ }with respect to $\D$. 
\end{cor}

\paragraph{\textbf{Remarks:}}
\begin{enumerate}
	\item Because the data distribution is $r$-separated, $k_n = 1$ will be $r$-consistent. Also observe that for $r$-separated distributions, $k_n = 1$ will converge towards the Bayes Optimal classifier.
	\item In general, $M$ converges towards the Bayes Optimal classifier provided that $k_n \to \infty$ in addition to $k_n /n \to 0$. This condition is not necessary for \rconsy -- because the distribution is $r$-separated. 
\end{enumerate}

We next show that kernel classifiers are also $r$-consistent on $r$-separated data distributions, provided the kernel function decreases rapidly enough. 

\begin{cor}\label{thm_kernel}
Let $W$ be a kernel classifier over $\X \times \Y$ constructed from $K$ and $h_n$. Suppose the following properties hold for $K$ and $h_n$.
\begin{enumerate}
	\item For any $c > 1$, $\lim_{x \to \infty} \frac{K(cx)}{K(x)} = 0.$
	\item $\lim_{n \to \infty} h_n = 0.$
\end{enumerate}
If $\D$ is an $r$-separated distribution over $\X \times \Y$, then $W$ is \rcons\emph{ }with respect to $\D$. 
\end{cor}

Observe that Condition 1. is satisfied for any $K(x)$ that decreases more rapidly than an inverse polynomial -- and is hence satisfied by most popular kernels like the Gaussian kernel. Is the condition on $K$ in Corollary~\ref{thm_kernel} necessary? The following example illustrates that a kernel classifier with any arbitrary $K$ is not necessarily $r$-consistent. This indicates that some sort of condition needs to be imposed on $K$ to ensure $r$-consistency; finding a tight necessary condition however is left for future work. 

 \paragraph{Example 3} Let $\X = [-1, 1]$ and let $\D$ be a distribution with $p_\D(-1, -1) = 0.1$ and $p_\D(1, 1) = 0.9$. Clearly, $\D$ is $0.3$-separated. Let $K(x) = e^{-\min(|x|, 0.2)^2}$. Let $h_n$ be any sequence with $\lim_{n \to \infty} h_n = 0$ and $\lim_{n \to \infty} nh_n = \infty$. Let $W$ be the weight classifier with input $S = \{(x_1, y_1), (x_2, y_2), \dots, (x_n, y_n)\}$ such that $$w_i^S(x) = \frac{K(\frac{|x- x_i|}{h_n})}{\sum_{j=1}^n K(\frac{|x-x_j|}{h_n})}.$$ $W$ can be shown to satisfy all the conditions of Theorem \ref{thm_stone} (the proof is analogous to the case for a Gaussian Classifier), and is therefore consistent. However, $W$ does not learn a robust classifier on $\D$ for $r = 0.3$. 

Consider $x = -0.7$. For any $\{(x_1, y_1), (x_2, y_2), \dots, (x_n, y_n)\} \sim \D^n$, all $x_i$ will either be $-1$ or $1$. Therefore, since $K(|x - (-1)|) = K(|x - 1|)$, it follows that $w_i^S(x) = \frac{1}{n}$ for all $1 \leq i \leq n$. Since $x_i = 1$ with probability $0.9$, it follows that with high probability $x$ will be classified as $1$ which means that $f$, the output of $W$, is not robust at $x = -1$. Thus $f$ has astuteness at most $0.9$ which means that $W$ is \textit{not} \rcons\ for $r=0.3$.  \\ \\

\section{General Distributions}\label{sec:general}

We next consider more general data distributions, where data from different classes may be close together in space, and may even overlap. Observe that unlike the $r$-separated case, here there may be no classifier with astuteness one. Thus, a natural question is: what does the optimally astute classifier look like, and how can we build non-parametric classifiers to this limit?

\subsection{The $r$-Optimal Classifier and Adversarial Pruning}

\cite{YRWC19} propose a large-sample limit -- called the $r$-optimal -- and show that it is analogous to the Bayes Optimal classifier for robustness. More specifically, given a data distribution $D$, to find the $r$-optimal classifier, we solve the following optimization problem.  

\begin{equation}\label{optim_prob}
\begin{split}
\max_{S_{+1}, S_{-1}} &\int_{x \in S_{+1}} p(y=+1|x)d\mu_{\D}(x) + \\
&\int_{x \in S_{-1}} p(y=-1|x)d\mu_{\D}(x) \\
&\text{ subject to } d(S_{+1}, S_{-1}) > 2r 
\end{split}
\end{equation}

Then, the $r$-optimal classifier is defined as follows. 

\begin{defn}
\cite{YRWC19} Fix $r, \D$. Let $S_{+1}^*$ and $S_{-1}^*$ be any optimizers of (\ref{optim_prob}). Then the $r$-optimal classifier, $\b_r$ is any classifier such that $\b_r(x) = j$ whenever $d(S_j^*, x) \leq r$. 
\end{defn}

\cite{YRWC19} show that the $r$-optimal classifier achieves the optimal astuteness -- out of all classifiers on the data distribution $\D$; hence, it is a robustness analogue to the Bayes Optimal Classifier. Therefore, for general distributions, the goal in robust classification is to find non-parametric algorithms that output classifiers that converge towards $\b_r$. 

To find robust classifiers, \cite{YRWC19} propose Adversarial Pruning -- a defense method that preprocesses the training data by making it better separated. More specifically, Adversarial Pruning takes as input a training dataset $S$ and a radius $r$, and finds the largest subset of the training set where differently labeled points are at least distance $2r$ apart.

\begin{defn}
A set $S_r \subset \X \times \Y$ is said to be \textbf{$r$-separated} if for all $(x_1, y_1), (x_2, y_2) \in S_r$, if $y_1 \neq y_2$, then $d(x_1, x_2) > 2r$. To \textbf{adversarially prune} a set $S$ is to return its largest $r$-separated subset. We let $\ap(S, r)$ denote the result of adversarially pruning $S$.  
\end{defn}

Once an $r$-separated subset $S_r$ of the training set is found, a standard non-parametric method is trained on $S_r$.  While~\cite{YRWC19} show good empirical performance of such algorithms, no formal guarantees are provided. We next formally characterize when adversarial pruning followed by a non-parametric method results in a classifier that is provably $r$-consistent.

Specifically, we consider analyzing the general algorithm provided in Algorithm \ref{alg:gen}.

\begin{algorithm}[tb]
   \caption{\ga}
   \label{alg:gen}
\begin{algorithmic}
   \STATE {\bfseries Input:} $S \sim \D^n$, weight function $W$, 
   robustness radius $r$
   \STATE $S_r \leftarrow \ap(S, r)$
   \STATE{\bfseries Output:} $W_{S_r}$
\end{algorithmic}
\end{algorithm}

\subsection{Convergence Guarantees}

We begin with some notation. For any weight function $W$ and radius $r > 0$, we let $\ga(W,r)$ represent the weight function that outputs weights for $S \sim \D^n$ according to $\ga(S, W, r)$. In particular, this can be used to convert any weight function algorithm into a new weight function which takes robustness into account. A natural question is, for which weight functions $W$ is $\ga(W,r)$ \rcons? Our next theorem provides sufficient conditions for this.

\begin{thm}\label{thm_weight_general}
Let $W$ be a weight function over $\X \times \Y$, and let $\D$ be a distribution over $\X \times \Y$. Fix $r >0$. Let $S_r = \ap(S, r)$.  For convenience, relabel $x_i, y_i$ so that $S_r = \{(x_1, y_1), (x_2, y_2), \dots, (x_m, y_m)\}$. Suppose that for any $0 < a < b,$ 
\begin{equation*}\label{condition}
\lim_{n \to \infty} \mathbb{E}_{S \sim \D^n}\big [ \frac{1}{m} \sum_{i = 1}^m \sup_{x \in B(x_i, a)} \sum_{j = 1}^m w_j^{S_r}(x)I_{||x_j - x|| > b} \big] = 0. 
\end{equation*}
Then $\ga(W,r)$ is \rcons\emph{ }with respect to $\D$. 
\end{thm}

\paragraph{\textbf{Remark}:}
There are two important differences between the conditions in Theorem \ref{thm_weight_general} and Theorem~\ref{thm_stone_cons}.
\begin{enumerate}
	\item We replace $S$ with $S_r$.
	\item The expectation over $X \sim \D_\X$ is replaced with an average over $\{x_1, x_2, \dots, x_m\}$. The intuition here is that we are replacing $\D$ with a uniform distribution over $S_r$. While $\D$ may not be $r$-separated, the uniform distribution over $S_r$ is, and represents the region of points where our classifier is astute. 
\end{enumerate}

A natural question is what satisfies the conditions in Theorem~\ref{thm_weight_general}. We next show that $k_n$-nearest neighbors and kernel classifiers with rapidly decaying kernel functions continue to satisfy the conditions in Theorem \ref{thm_weight_general}; this means that these classifiers, when combined with Adversarial Pruning, will converge to $r$-optimal classifiers in the large sample limit.

\begin{cor}\label{thm_NN_gen}
Let $k_n$ be a sequence with $\lim_{n \to \infty} \frac{k_n}{n} = 0$, and let $M$ denote the $k_n$-nearest neighbor algorithm. Then for any $r > 0$, $\ga(M, r)$ is \rcons.
\end{cor}
\paragraph{\textbf{Remark}:} Corollary \ref{thm_NN_gen} gives a formal guarantee in the large sample limit for the modified nearest-neighbor algorithm proposed by \cite{YRWC19}.

\begin{cor}\label{thm_kern_gen}
Let $W$ be a kernel classifier over $\X \times \Y$ constructed from $K$ and $h_n$. Suppose the following properties hold for $K$ and $h_n$.
\begin{enumerate}
	\item For any $c > 1$, $\lim_{x \to \infty} \frac{K(cx)}{K(x)} = 0.$
	\item $\lim_{n \to \infty} h_n = 0.$
\end{enumerate}
Then for any $r > 0$, $\ga(W, r)$ is \rcons.
\end{cor}

Observe again that Condition 1. is satisfied by any $K$ that decreases more rapidly than an inverse polynomial kernel; it is thus satisfied by most popular kernels, such as the Gaussian kernel. 

\section{Validation}
\begin{figure*}[ht]
\vskip 0.2in
\begin{center}
\subfloat[][Noiseless Histogram]{\includegraphics[width=.29\textwidth]{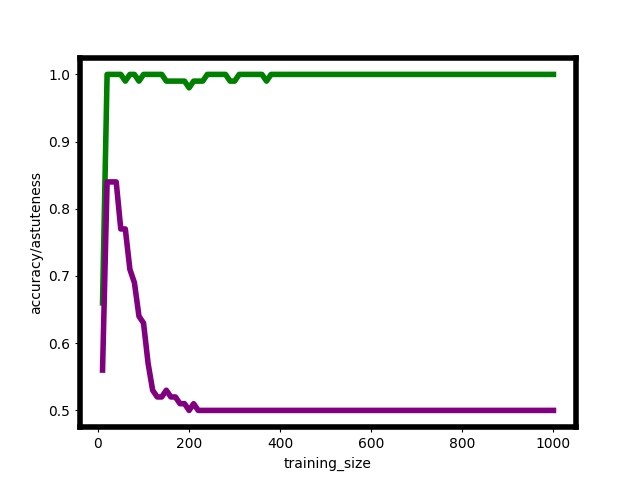}}\quad
   \subfloat[][Noisy Histogram]{\includegraphics[width=.29\textwidth]{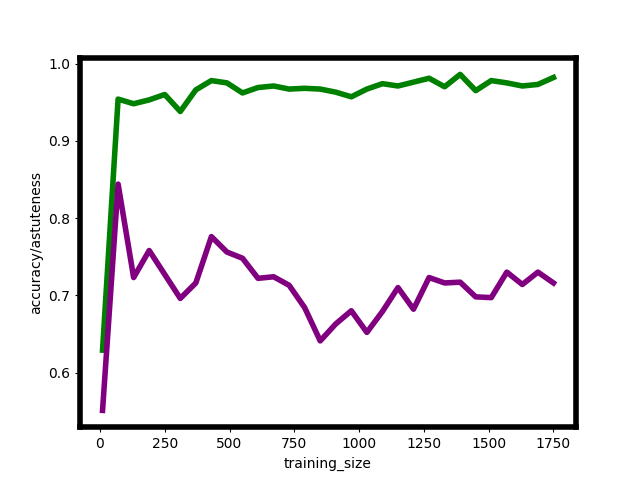}} \quad
   \subfloat[][Histogram trained on 500 samples]{\includegraphics[width=.29\textwidth]{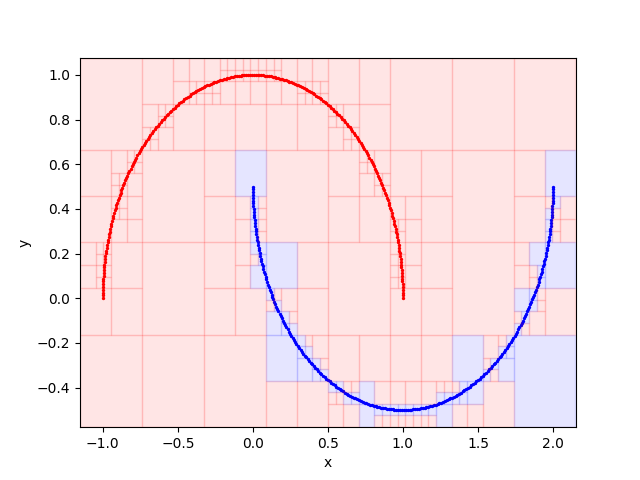}}\\
   \subfloat[][Noiseless 1-NN]{\includegraphics[width=.29\textwidth]{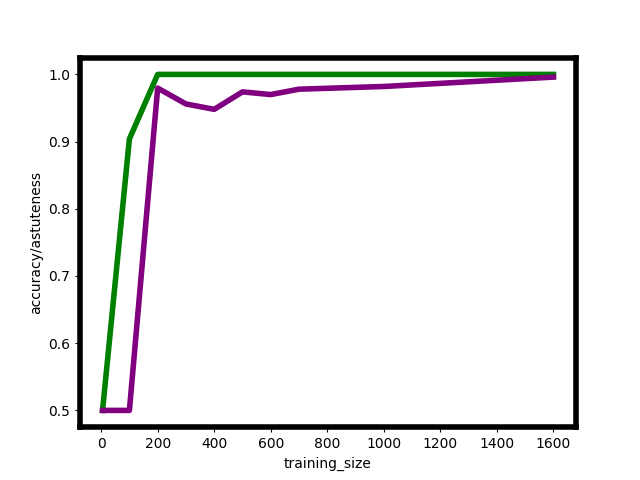}} \quad
   \subfloat[][Noisy 1-NN]{\includegraphics[width=.29\textwidth]{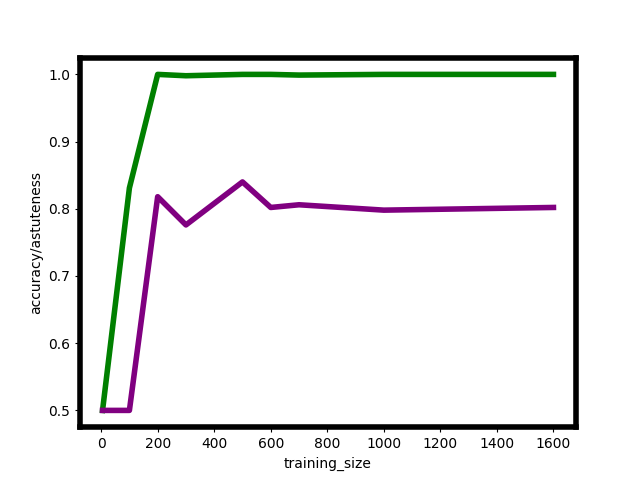}}\quad
   \subfloat[][Histogram trained on 3000 samples]{\includegraphics[width=.29\textwidth]{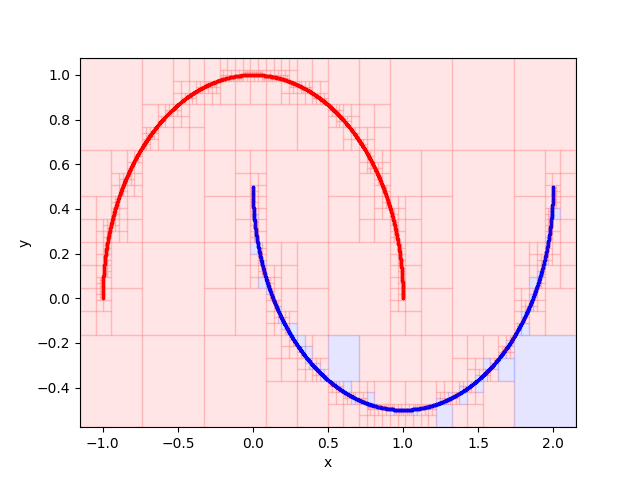}}
\end{center}
\caption{Empirical accuracy/astuteness of different classifiers as a function of training sample size. Accuracy is shown in green, astuteness in purple. Left : Noiseless Setting. Right: Noisy Setting. Top Row: Histogram Classifier, Bottom Row: 1-Nearest Neighbor}
\label{fig:1}
\vskip -0.2in
\end{figure*}

Our theoretical results are, by nature, large sample; we next validate how well they apply to the finite sample case by trying them out on a simple example. In particular, we ask the following question:

\begin{quote}
How does the robustness of non-parametric classifiers change with increasing sample size?
\end{quote}

This question is considered in the context of two simple non-parametric classifiers -- one nearest neighbor (which is guaranteed to be $r$-consistent) and histograms (which is not). To be able to measure performance with increasing data size, we look at a simple synthetic dataset -- the Half Moons. 

\subsection{Experimental Setup}

\paragraph{Classifiers and Dataset.} We consider two different classification algorithms -- one nearest neighbor (NN) and a Histogram Classifier (HC).  We use the Halfmoon dataset with two settings of the gaussian noise parameter $\sigma$, $\sigma = 0$ (Noiseless) and $\sigma =0.08$ (Noisy). For the Noiseless setting, observe that the data is already $0.1$-separated; for the Noisy setting, we use Adversarial Pruning (Algorithm~\ref{alg:gen}) with parameter $r = 0.1$ for both classification methods.

\paragraph{Performance Measure.} We evaluate robustness with respect to the $\ell_{\infty}$ metric, that is commonly used in the adversarial examples literature. Specifically, for each classifier, we calculate the {\em{empirical astuteness}}, which is the fraction of test examples on which it is astute.

Observe that computing the empirical astuteness of a classifier around an input $x$ amounts to finding the adversarial example that is {\em{closest to}} $x$ according to the $\ell_{\infty}$ norm. For the $1$-nearest neighbor, we do this using the optimal attack algorithm proposed by Yang et. al.~\cite{YRWC19}. For the histogram classifier, we use the optimal attack framework proposed by~\cite{YRWC19}, and show that the structure of the classifier can be exploited to solve the convex program efficiently. Details are in Appendix C.

We use an attack radius of $r = 0.1$ for the Noiseless setting, and $r = 0.09$ for the Noisy setting. For all classification algorithms, we plot the empirical astuteness as a function of the training set size. As a baseline, we also plot their standard accuracy on the test set. 

\subsection{Results}

The results are presented in Figure~\ref{fig:1}; the left two panels are for the Noiseless setting while the two center ones are for the Noisy setting.  

The results show that as predicted by our theory, for the Noiseless setting, the empirical astuteness of nearest neighbors converges to $1$ as the training set grows. For Histogram Classifiers, the astuteness converges to $0.5$ -- indicating that the classifier may grow less and less astute with higher sample size even for well-separated data. This is plausibly because the cell size induced by the histogram grows smaller with growing training data; thus, the classifier that outputs the default label $-1$ in empty cells is incorrect on adversarial examples that are close to a point with $+1$ label, but belongs to a different, empty cell. The rightmost panels in Figure~\ref{fig:1} provide a visual illustration of this process. 

For the Noisy setting, the empirical astuteness of adversarial pruning followed by nearest neighbors converges to $0.8$. For histograms with adversarial pruning, the astuteness converges to $0.7$, which is higher than the noiseless case but still clearly sub-optimal.

\subsection{Discussion}

Our results show that even though our theory is asymptotic, our predictions continue to be relevant in finite sample regimes. In particular, on well-separated data, nearest neighbors that we theoretically predict to be intrinsically robust is robust; histogram classifiers, which do not satisfy the conditions in Theorem~\ref{thm_stone_cons} are not. Our predictions continue to hold for data that is not well-separated. Nearest neighbors coupled with Adversarial
Pruning continues to be robust with growing sample size, while histograms continue to be non-robust. Thus our theory is confirmed by practice.

\section{Conclusion}

In conclusion, we rigorously analyze when non-parametric methods provide classifiers that are robust in the large sample limit. We provide a general condition that characterizes when non-parametric methods are robust on well-separated data, and show that Adversarial Pruning of~\cite{YRWC19} works on data that is not well-separated. 

Our results serve to provide a set of guidelines that can be used for designing non-parametric methods that are robust and accurate on well-separated data; additionally, we demonstrate that when data is not well-separated, preprocessing by adversarial pruning~\cite{YRWC19} does lead to optimally astute solutions in the large sample limit.

\section*{Acknowledgements}

We thank NSF under CNS 1804829  for research support.

\bibliography{refs}
\bibliographystyle{icml2020}

\newpage

\onecolumn

\appendix

\section{Proofs for $r$-separated distributions}

For any distribution $\D$ over $\X \times Y$, it will be convenient to use the following notation: for any measurable $S \subset \X$, let $\P_\D[S] = \P_{(x,y) \sim \D}[x \in S]$. The following definition will be central to our proofs. 

\begin{defn}
Let $\D$ be a distribution over $\X \times Y$. An \textbf{$(\epsilon, \gamma, \alpha)$-decomposition} of $\D$ is a finite set of closed balls $B_1, B_2, \dots, B_s \subset \X$ each with radius $\gamma$ such that $$\P_\D[\cup_1^s B_i] > 1 - \epsilon,$$ and such that $\P_\D[B_i] \geq \alpha > 0$ for $1 \leq i \leq s$. 
\end{defn}

\begin{lem}\label{lem_balls}
Let $\X$ be a totally bounded metric space. For any distribution $\D$, and $\epsilon, \gamma > 0$, there exists $\alpha > 0$ such that $\D$ admits a $(\epsilon, \gamma, \alpha)$-decomposition. 
\end{lem}

\begin{proof}
Fix any $x \in \X$ and $\epsilon, \gamma > 0$. Then the sequence of balls $\{S_i = B(x, i)\}$ has union equal to $\X$. Therefore, there exists $j$ such that $P_\D(S_j) > 1 - \epsilon$. Since $S_j$ is totally bounded and complete, it is compact. Let $B^o(x, a)$ denote the open ball centered at $x$ with radius $a$. Therefore, taking an open cover of $S_j$, $\{B^o(x, \gamma): x \in S_j\}$, we can take a finite subcover $\{B_1^o, B_2,^o, \dots, B_t^o\}$ that cover $S_j$. Discarding balls such that $\P_\D(B_i^o) = 0$ and taking the closure of each ball gives the desired result, with $\alpha = \min_{i}P_\D(B_i)$.  
\end{proof}

To prove Theorem \ref{thm_stone_cons}, we use the following lemma. 

\begin{lem}\label{lem_expectation}
Let $\D$ be a distribution over $\X \times \Y$, and let $B_1, B_2, \dots, B_s$ be a $(\epsilon, \gamma, \alpha)$-decomposition of $\D$, and let $r > 3\gamma$. If $W$ is a weight function satisfying the conditions of Theorem \ref{thm_stone_cons}, then for any $\delta > 0$ there exists $N$ such that for $n \geq N$, with probability $1-\delta$ over $S \sim \D^n$, and $w_1, w_2, \dots, w_n$ learned by $W$ from $S$, $$\sup_{\{x: d(x, \cup_1^s B_i) \leq r - 3\gamma\}} \sum_1^n w_i(x)I_{d(x_i, x) > r} < \frac{1}{3}.$$
\end{lem}

\begin{proof} 
Fix $\delta > 0$, and let $Y$ be the indicator variable defined as $$Y = \begin{cases} 1 & \text{ if }\sup_{\{x: d(x, \cup_1^s B_i) \leq r - 3\gamma\}} \sum_1^n w_i(x)I_{d(x_i, x) > r} \geq \frac{1}{3} \\ 0 & \text{ if }\sup_{\{x: d(x, \cup_1^s B_i) \leq r - 3\gamma\}} \sum_1^n w_i(x)I_{d(x_i, x)> r} < \frac{1}{3} \end{cases}.$$ It suffices to show that there exists $N$ such that for all $n \geq N$, $E_{S \sim \D}[Y] \leq \delta$. 

Fix $S \sim \D^n$ and suppose that $Y = 1$. Then there exists $x^*, B_i^*$ such that $d(x^*, B_i^*) \leq r - 3\gamma$ and such that $$\sum_1^n w_i(x^*)I_{d(x_i, x^*) > r} \geq \frac{1}{3}.$$ By definition, $B_i$ has radius $\gamma$, so by the triangle inequality, for any $x \in B_i^*$, $d(x, x^*) \leq 2\gamma + r - 3\gamma = r - \gamma$. This implies $x^* \in B(x, r-\gamma)$. Therefore, for any $x \in B_i^*$, $$\sup_{x' \in B(x, r-\gamma)} \sum_1^n w_i(x')I_{d(x', x_i) > r} \geq \sum_1^n w_i(x^*)I_{d(x^*, x_i) > r} \geq \frac{1}{3}.$$ By the definition of an $(\epsilon, \gamma, \alpha)$-decomposition, we have that $P_\D(B_i^*) \geq \alpha$. As a consequence, we have that $$\mathbb{E}_{X \sim \D_\X} \big [ \sup_{x' \in B(X, r-\gamma)} \sum_1^n w_i(x')I_{||x_i - x'|| > r} \big] \geq P_\D[B_i^*]\frac{1}{3} \geq \frac{\alpha}{3}.$$ Since the previous inequality is guaranteed to hold if $Y = 1$, taking the expectation over $S$ yields that $$\mathbb{E}_{S \sim \D^n} \mathbb{E}_{X \sim \D_\X} \big [ \sup_{x' \in B(X, r-\gamma)} \sum_1^n w_i(x')I_{||x_i - x'|| > r} \big] \geq \frac{\alpha E[Y]}{3}.$$ By the conditions of Theorem \ref{thm_stone_cons}, the left side of the equation must tend to $0$ as $n \to \infty$. This implies that the same must hold for the right side. Therefore, $E[Y]$ tends to $0$ as $n \to \infty$, and we can select $N$ such that $E[Y] < \delta$ for $n \geq n$, which completes the proof. 
\end{proof}

\begin{proof} (\textbf{Theorem \ref{thm_stone_cons}})
Let $W$ be a weight function that satisfies the condition of Theorem \ref{thm_stone_cons}. Fix $\epsilon, \delta > 0$, and $\gamma < r/3$. Applying Lemma \ref{lem_balls}, let $B_1, B_2, \dots, B_s$ be an $(\epsilon, \gamma, \alpha)$-decomposition of $\D$. Let $T^+$ and $T^-$ be subsets of $\X$ corresponding to the definition of $r$-separation for $\D$.  

For $S \sim \D^n$, let $A$ denote the event that $$\sup_{\{x: d(x, \cup_1^s B_i) \leq r - 3\gamma\}} \sum_1^n w_i(x)I_{d(x_i, x) > r} < \frac{1}{3}.$$
Suppose $A$ holds. Pick a $B_i$. Since $T^+$ and $T^-$ have distance greater than $2r$, and $diam(B_i) \leq 2\gamma < r$, either $B_i \cap T^+ = \emptyset$ or $B_i \cap T^- = \emptyset$. Note that for $n$ sufficiently large, both cannot be empty since $P_\D(B_i) \geq \alpha > 0$ and each $x$ in the support of $\D$ is either in $T^+$ or $T^-$. 

Without loss of generality, $B_i \cap T^- = \emptyset$. Then $B_i \cap T^+ \neq \emptyset$. $B_i$ has diameter $2\gamma$. Thus $d(B_i, T^-) > 2r - 2\gamma$. Let $x \in B(B_i, r-3\gamma)$. Then if $(x_j, -) \in S$, by the triangle inequality, $d(x, x_j) > 2r -2\gamma - (r - 3\gamma) = r+\gamma$. 

Substituting this and using event $A$, we have that $$\sum_1^n w_i^S(x)I_{(x_i, -) \in S} \leq \sum_1^n w_i^S(x)I_{d(x_i, x) > r} < \frac{1}{3}.$$ It follows that $W_S(x) = +1$. An analogous argument holds for $B_i \cap T^+ = \emptyset$. This implies that $W_S$ is astute with radius $r-3\gamma$ over all $B_i$.

$\cup B_i$ has measure at least $1-\epsilon$. By Lemma \ref{lem_expectation}, for any $\delta>0$ event $A$ holds with probability $1-\delta$ for $n$ sufficiently large. Therefore, for $n$ sufficiently large, we see that $A_{r-3\gamma}(W_S, \D) \geq 1-\epsilon$ with probabiltiy $1-\delta$. Because $\epsilon, \delta$ and $\gamma$ were arbitrary, it follows that $W$ is \rcons,\emph{ }as desired.

\end{proof}

\begin{proof} (\textbf{Corollary \ref{nn_sep_thm}})
For any $S = \{(x_1, y_1), (x_2, y_2), \dots, (x_n, y_n)\} \subset \X \times \Y$, let $w_i^S(x)$ be $1$ if and only if $x_i$ is one of the $k_n$ nearest neighbors of $x$ in the set $S_\X = \{x_1, x_2, \dots x_n\}$. Let $\D$ be a distribution over $\X \times \Y$. By Theorem \ref{thm_stone_cons}, it suffices to show that for any $0 < a < b$, $$\lim_{n \to \infty} \mathbb{E}_{X \sim \D_\X}[\mathbb{E}_{S \sim \D^n} [\sup_{x' \in B(x,a)} \sum_1^n w_i^S(x')I_{d(x_i, x') > b}]] = 0.$$ Fix $0 < a < b$, and let $\epsilon > 0$. 

Pick $\gamma > 0$ such that $a+ 2\gamma < b$. This is possible for any $a < b$. Let $B_1, B_2, \dots, B_s$ be an $(\epsilon, \gamma, \alpha)$-decomposition of $\D$. By applying a Chernoff bound followed by a union bound, for any $\delta > 0$ there exists $n$ such that with probability $1-\delta$ over $S \sim \D^n$, each $B_i$ satisfies $|B_i \cap S_\X| \geq \frac{n\alpha}{2}$. Furthermore, if $n$ is sufficiently large, then $\frac{n\alpha}{2} > k_n$ holds as well. 

Consider any $x \in B_i$, and$x' \in B(x,a)$. $B_i$ has radius $\gamma$ and also satisfies $|B_i \cap S_\X| > k_n$. Therefore, there are at least $k_n$ points within distance $a+2\gamma$ of $x$. Because $a + 2\gamma < b$, it follows that none of the $k_n$ nearest neighbors of $x'$ can have distance more than $b$ from $x'$. In particular, $$\sum_1^n w_i^S(x')I_{d(x_i, x') > b} = 0.$$ Since $B_i$, $x$ and $x'$ were arbitrary, we have that for all $x \in \cup B_i$, $$\sup_{x' \in B(x,a)} \sum_1^n w_i^S(x')I_{d(x_i, x') > b} \leq  \begin{cases} 0 & |B_i \cap S_\X| \geq \frac{n\alpha}{2}, 1 \leq i \leq s \\ 1 & \text{otherwise} \end{cases}$$

Since $X \in \cup_1^s B_i$ with probability at least $1-\epsilon$, and since $|B_i \cap S_\X| \geq \frac{n\alpha}{2}, 1 \leq i \leq s$ with probability at least $1-\delta$, it follows that $$\mathbb{E}_{X \sim \D}[\mathbb{E}_{S \sim \D^n} [\sup_{x' \in B(x,a)} \sum_1^n w_i^S(x')I_{d(x_i, x') > b}]] \leq (1- \delta - \epsilon)0 + \delta + \epsilon = \delta + \epsilon,$$ which can be made arbitrarily small as $\epsilon$ and $\delta$ were arbitrary. Therefore, the limit as $n$ approaches infinity is $0$, as desired.
\end{proof}

\begin{proof} (\textbf{Corollary \ref{thm_kernel}})
Let $\D$ be a distribution over $\X \times \Y$. By Theorem \ref{thm_stone_cons}, it suffices to show that for any $0 < a < b$, $$\lim_{n \to \infty} \mathbb{E}_{X \sim \D}[\mathbb{E}_{S \sim \D^n} [\sup_{x' \in B(x,a)} \sum_1^n w_i^S(x')I_{d(x_i, x') > b}]] = 0.$$ Fix $0 < a < b$, and let $\epsilon > 0$. 

Pick $\gamma > 0$ be such that $a+ 2\gamma < b$. Let $B_1, B_2, \dots, B_s$ be an $(\epsilon, \gamma, \alpha)$-decomposition of $\D$. By applying a Chernoff bound, for any $\delta > 0$ there exists $n$ such that with probability $1-\delta$ over $S \sim \D^n$, each $B_i$ satisfies $|B_i \cap S_\X| \geq \frac{n\alpha}{2}$.

Next, consider any $x_i, x_j \in S_\X$, and let $x$ be a point such that $d(x_i, x) \leq a+2\gamma$ and $d(x_j, x) > b$. Then we have that 
\begin{equation*}
\begin{split}
\frac{w_j^S(x)}{w_i^S(x)} = \frac{K(\frac{d(x_j, x)}{h_n})}{K(\frac{d(x_i, x)}{h_n})}.
\end{split}
\end{equation*}
Because $b > a + 2\gamma$, $\frac{d(x_j, x)}{d(x_i, x)} > 1$. Therefore, since $\lim_{n \to \infty} h_n = 0$ and $\lim_{x \to \infty} \frac{K(cx)}{K(x)} = 0$ for $c > 1$, it follows that for any $\beta > 0$, there exists $N$ such that for $n \geq N$, $$\frac{w_j^S(x)}{w_i^S(x)} \leq \frac{\alpha\beta}{2}.$$ 

Fix any such $\beta$, and consider any $x$ with $d(x, B_i) \leq a$. Then $d(x, x') \leq a+ 2\gamma < b$ for any $x' \in B_i$. Recall that $B_i$ contains at least $\frac{n\alpha}{2}$ points, and let $c = \min_{i, d(x_i, x) \leq a + 2\gamma} w_i(x)$. Then it follows that 
\begin{equation*}
\begin{split}
\sum_1^n w_i^S(x)I_{d(x_i, x) > b} &\stackrel{(a)}{=} \frac{\sum_1^n w_i^S(x)I_{d(x_i, x) > b}}{\sum_1^n w_i^S(x)} \\
&\stackrel{(b)}{\leq} \frac{\sum_1^n w_i^S(x)I_{d(x_i, x) > b}}{\sum_1^n w_i^S(x)I_{d(x_i, x) \leq a+2\gamma}} \\
&\stackrel{(c)}{\leq} \frac{nc\frac{\alpha\beta}{2}}{\frac{n\alpha}{2}c} \\
&= \beta
\end{split}
\end{equation*} $(a)$ holds because the weights always sum to $1$. $(b)$ holds because we are reducing the denominator. $(c)$ holds because there are at least $\frac{n\alpha}{2}$ points in $B_i$, with $c$ being the minimum weight (stated above). The numerator is a result of the inequality shown above in which $w_j^S(x)/w_i^S(x) \leq \alpha\beta/2$ if $d(x_j, x) > b$ and $d(x_i, x) \leq a+2\gamma$.

Using this, we get the following bound: $$\sup_{x' \in B(X,a)} \sum_1^n w_i^S(x')I_{d(x_i, x') > b} \leq  \begin{cases} \beta & x \in \cup_1^s B_i, |B_i \cap S_\X| \geq \frac{n\alpha}{2}, 1 \leq i \leq s \\ 1 & \text{otherwise} \end{cases}$$ 

Since $x \in \cup_1^s B_i$ with probability $1-\epsilon$, and since $|B_i \cap S_\X| \geq \frac{n\alpha}{2}, 1 \leq i \leq s$ with probability $1-\delta$, it follows that $$\mathbb{E}_{X \sim \D}[\mathbb{E}_{S \sim \D^n} [\sup_{x' \in B(x,a)} \sum_1^n w_i^S(x')I_{d(x_i, x') > b}]] \leq (1- \delta - \epsilon)\beta + \delta + \epsilon.$$ which can be made arbitrarily small as $\epsilon, \beta,$ and $\delta$ were arbitrary. Therefore, the limit as $n$ approaches infinity is $0$, as desired. 
\end{proof}

\section{Proofs for general distributions}

\begin{lem}\label{chernoff_max_lem}
Let $B_1, \dots, B_s$ be a $(\epsilon, \alpha, \gamma)$ decomposition of $\D$ over $\X \times \Y$. Let $U \subseteq [s]$. Then if $n \geq O(\frac{s2^{2s}\log(1/\delta)}{\epsilon^2})$, then with probability at least $1-\delta$, for all $U$ we have: $$|\P_{(x,y) \sim \D}[x \in \cup_{i \in U} B_i, y = +] - \P_{(x,y) \sim \D_S}[x \in \cup_{i \in U} B_i, y = +]| \leq \epsilon,$$$$|\P_{(x,y) \sim \D}[x \in \cup_{i \in U} B_i, y = -] - \P_{(x,y) \sim \D_S}[x \in \cup_{i \in U} B_i, y = -]| \leq \epsilon.$$ 
\end{lem}
\begin{proof}
For any given $U \subseteq [s]$, by a Chernoff bound we have that $$|\P_{(x,y) \sim \D}[x \in \cup_{i \in U} B_i, y = +] - \P_{(x,y) \sim \D_S}[x \in \cup_{i \in U} B_i, y = +]| > \epsilon$$ with probability at most $\frac{\delta}{2^{s+1}}$. Taking a union bound over all $U$, we see that with probability $1-\frac{\delta}{2}$, $$|\P_{(x,y) \sim \D}[x \in \cup_{i \in U} B_i, y = +] - \P_{(x,y) \sim \D_S}[x \in \cup_{i \in U} B_i, y = +]| \leq \epsilon$$ for all $U \subseteq [m]$. Applying the same to $y = -1$ and taking a union bound implies the result.
\end{proof}

\begin{lem}\label{gen_thm}
Let $M$ be a classification algorithm over $\X \times \Y$, $r >0$ be a radius, and $\D$ be a distribution over $\X \times \Y$. Then for any $\epsilon, \delta$ over $(0,1)$, and for all $\gamma$ over $(0, r/2)$, there exists $N$ such that for $n \geq N$, with probability $1-\delta$ over $S \sim \D^n$, $$A_{r - \gamma}(M_S, \D) \geq A_r(M_S, \D_S) - \epsilon,$$ where $\D_S$ denotes the uniform distribution over $S$.  
\end{lem}

\begin{proof} (\textbf{Lemma \ref{gen_thm}})
Fix $\epsilon, \delta > 0$ and $\gamma < r/2$. Applying Lemma \ref{lem_balls}, let $B_1, \dots, B_s$ be a $(\epsilon, \alpha, \gamma)$ decomposition of $\D$. 

Let $T$ be the subset of $S$ such that $M_S$ is astute at $T$ with radius $r$. Define: $$I_T^+ = \{i| (x_j, +) \in T, x_j \in B_i\}$$$$I_T^- = \{i| (x_j, -) \in T, x_j \in B_i\}.$$

Observe that $I_T^+ \cap I_T^- = \emptyset$. To see this, notice that $B_i$ has radius $\gamma < r/2$. This implies that any $(x_j, +), (x_k, -) \in B_i$ would force $M_S$ to not be astute at either of those points. Thus we an think of $I_T^+$ being the set of positively labeled balls, and $I_T^{-}$ being the set of negatively labeled balls.

Let $B^+ = \cup_{i \in I_T^+} B_i$ and $B^- = \cup_{i \in I_T^-} B_i$. Our strategy will be to argue that $M_S$ must be robust with radius $r-2\gamma$ at $B^+ \cup B^-$, and then to observe that $\P_\D[(B^+, +)] + \P_\D[(B^-, -)]$ must be close to $A_r(M_S, \D_S)$. 

Let $T_\X \subset \X$ denote the set of all $x_i$ such that $(x_i, y_i) \in T$. By the definitions of $\D_S$ and $T$, we have that 
\begin{equation*}
\begin{split}
A_r(M_S, \D_S) &= \frac{|T|}{n} \\
&= \frac{|T_\X \cap B^+|}{n} + \frac{|T_\X \cap B^-|}{n} + \frac{|T_\X \setminus (B^+ \cup B^-)|}{n}.
\end{split}
\end{equation*}

If $x_i \in \cup_1^s B_j$ and $x_i \in T_\X$, then by definition, $x \in (B^+ \cup B^-)$. Therefore, $T_\X \setminus (B^+ \cup B^-)$ consists of $x_i \notin \cup_1^s B_j$. Using this, we see that 
\begin{equation*}
\begin{split}
A_r(M_S, \D_S) &= \frac{|T_\X \cap B^+|}{n} + \frac{|T_\X \cap B^-|}{n} + \frac{|T_\X \setminus (B^+ \cup B^-)|}{n} \\
&\leq \P_{(x,y) \sim \D_S}[x \in B^+, y=+]+\P_{(x,y) \sim \D_S}[x \in B^-, y=-] + \P_{(x,y) \sim \D_S}[x \notin \cup_1^s B_j].
\end{split}
\end{equation*}

If $n$ is sufficiently large, then by Lemma \ref{chernoff_max_lem}, each term on the right is within $\epsilon$ of its corresponding probability over $\D$. Thus we see that with probability $1-\delta$, 
\begin{equation}\label{astute_bound_eqn}
A_r(M_S, \D_S) \leq \P_{(x,y) \sim \D}[x \in \cup_{i \in I_T^+}B_i, y=+]+\P_{(x,y) \sim \D}[x \in \cup_{i \in I_T^-}, y=-] + 4\epsilon.
\end{equation} 

Observe that if $M_S$ is robust with radius $r$ at $x_j \in B_i$, then it is robust with radius $r-2\gamma$ at all $x \in B_i$. Furthermore, for $x_j \in \cup_{i \in I_T^+}B_i$, $M_S$ is astute at $(x_j, +1)$ with radius $r$. Therefore $M_S(x) = +1$ for all $x \in \cup_{i \in I_T^+}B_i$. Consequently, 
\begin{equation*}
\begin{split}
A_{r-2\gamma}(M_S, \D) &\geq \P_{(x,y) \sim \D}[x \in \cup_{i \in I_T^+}B_i, y=+]+\P_{(x,y) \sim \D}[x \in \cup_{i \in I_T^-}B_i, y=-] \\
&\geq A_r(M_S, \D_S) - 4\epsilon \text{ }(\text{by equation }\ref{astute_bound_eqn}).
\end{split}
\end{equation*}
Since this equation holds with probability $1 - \delta$, and since $\epsilon$ and $\gamma$ were arbitrary, the result follows. 
\end{proof}

\begin{proof}(\textbf{Theorem \ref{thm_weight_general}})
For convenience, we let $W'$ represent the weight function described by $\ga(S,W, r)$. In particular, $W'_S$ and $W_{S_r}$ are the same classifier, where $S_r$ denotes the largest $r$-separated subset of $S$.

Fix $\epsilon, \delta >0$, and let $0 < \gamma < r$. For convenience, let $$Z_ i = \sup_{x \in B(x_i, r-\gamma)} \sum_{j=1}^m w_j^{S_r}(x)I_{||x_j - x|| > r}.$$ Because $W$ fulfills the conditions of Theorem \ref{thm_weight_general}, there exists $N$ such that for $n > N$, with probability $1-\delta$ over $S \sim \D^n$, $ \frac{1}{m} \sum_{i = 1}^m Z_i < \epsilon.$ Therefore, there exist at most $3m\epsilon$ values of $i$ for which $Z_i > \frac{1}{3}$. 

Since $S_r$ is $r$-separated, it follows that $$\sup_{x \in B(x_i, r-\gamma)} \sum_1^m w_j^{S_r}(x)I_{y_j \neq y_i} \leq Z_i.$$ Consequently, if $Z_i \leq \frac{1}{3}$, then $W_{S_r}(x) = y_i$ for all $x \in B(x_i, r-\gamma)$. Let $\D_S$ denote the uniform distribution over $S$. Then we have that $$A_{r-\gamma}(W'_S, \D_S) = A_{r-\gamma}(W_{S_r}, \D_S) \geq \frac{|S_r|}{n} - 3\epsilon.$$ 
Observe that for $n$ sufficiently large, with probability $1-\delta$, $|A_r(\b_r, \D) - A_r(\b_r, \D_S)| \leq \epsilon$. The maximum possible astuteness over $\D_S$ is $\frac{|S_r|}{n}$ since no classifier can be astute at 2 oppositely labeled points with distance at most $2r$. Therefore, with probability $1-2\delta$, $$A_{r - \gamma}(W'_S, \D_S) \geq A_r(\b_r, \D) - 4\epsilon.$$ By Lemma \ref{gen_thm}, for $n$ sufficiently large, with probability $1-\delta$ $$A_{r-2\gamma}(W'_S, \D) \geq A_{r-\gamma}(W'_S, \D_S) - \epsilon.$$ Therefore, for $n$ sufficiently large, with probability $1-3\delta$ over $S \sim \D$, $$A_{r - 2\gamma}(W'_S, \D) \geq A_r(\b_r, \D) - 5\epsilon.$$ Since $\epsilon, \delta,$ and $\gamma$ were arbitrary, we are done. 
\end{proof}

The following two quick lemmas are used for the proofs of Corollaries \ref{thm_NN_gen} and \ref{thm_kern_gen}.

\begin{lem}\label{lem_point_count}
Let $B_1, B_2, \dots, B_s \subset \X$ denote $s$ balls. Let $T \subset \X$ satisfy $|T \cap \cup_1^s B_i| = m$. Let $$I_k \subseteq [s] = \{i: |B_i \cap T| \geq k\}.$$ Then $|\cup_{i \in I_k} B_i \cap T| \geq m - ks$. 
\end{lem}

\begin{proof}
For any $j \notin I_k$, $|B_j \cap T| < k$. Since there are at most $s$ such $j$, it follows that $|\cup_{i \notin I_k} B_i \cap T| < ks$. Taking the complement implies the result. 
\end{proof}

\begin{lem}\label{lem_half}
Let $S$ be a finite subset of $\X \times \Y$. For any $r > 0$, let $S_r$ denote the largest $r$-separated subset of $S$. Then $|S_r| \geq \frac{|S|}{2}$. 
\end{lem}

\begin{proof}
Let $S = \{(x_1, y_1), (x_2, y_2), \dots (x_n, y_n)\}$. 
Define: $$S_+ = \{(x_i, y_i): y_i = +1\}$$ $$S_- = \{(x_i, y_i): y_i = -1\}.$$ Observe that $S_+$ and $S_-$ are both $r$-separated and have union $S$. Therefore one must have cardinality at least $\frac{|S|}{2}$, which implies the same about $|S_r|$.
\end{proof}

\begin{proof}(\textbf{Corollary  \ref{thm_NN_gen}}) 
For convenience, we let $W'$ represent the weight function described by $\ga(S,W, r)$. In particular, $W'_S$ and $W_{S_r}$ are the same classifier, where $S_r$ denotes the largest $r$-separated subset of $S$.

Relabel the points in $S$ so that $$S_r = \{(x_1, y_1), (x_2, y_2), \dots, (x_m, y_m)\},$$ with $ m \leq n$. We will also let $S_r^\X = \{x_1, x_2, \dots, x_m\}$. 

By Theorem \ref{thm_weight_general}, it suffices to show that for any $0 < a < b$, $$\lim_{n \to \infty}\mathbb{E}_{S \sim \D^n} [\frac{1}{m} \sum_{i=1}^{m}\sup_{x \in B(x_i,a)} \sum_{j=1}^{m} w_j^{S_r}(x)I_{d(x_i, x) > b}] = 0,$$ where $w_j$ denote the weight functions corresponding to $W$. Fix $0 < a < b$, and let $\epsilon > 0$. 

Pick $\gamma > 0$ be such that $a+ 2\gamma < b$. Let $B_1, B_2, \dots, B_s$ be a $(\epsilon, \gamma, \alpha)$ decomposition of $\D$. By applying a Chernoff bound, for any $\delta > 0$ there exists $n_0$ such that for $n \geq n_0$, with probability $1-\delta$ over $S \sim \D^n$, $$|S_\X \cap \cup_1^s B_i| \geq (1-2\epsilon)n.$$  By Lemma \ref{lem_half}, $\frac{m}{n} \geq \frac{1}{2}$. It follows that $|S_r^\X \cap \cup_1^s B_i| \geq m(1-4\epsilon)$. 

Let $$J = \{i: |B_i \cap S_r^\X| \geq m\frac{\epsilon}{s}\}.$$ By Lemma \ref{lem_point_count} it follows that  $|S_r^\X \cap \cup_{i \in J} B_i| \geq m(1 - 4\epsilon) - m\epsilon = m(1 - 5\epsilon)$.

Next, observe that if $n$ is sufficiently large, then $$\frac{k_n}{m} \leq \frac{2k_n}{n} \leq \frac{\epsilon}{s}.$$ Therefore, $|B_i \cap S_r^X|_r \geq k_n$ for $i \in J$.

Fix any $B_j$ with $j \in J$, and consider $x$ with $d(x, B_j) \leq a$. Then $d(x, x') \leq a+ 2\gamma < b$ for any $x' \in B_j$. Therefore, since $|S_r^X \cap B_i| \geq k_n$, all $k_n$-nearest neighbors of $x$ have distance at most $b$ to $x$. This implies that $$\sum_1^m w_i^{S_r}(x)I_{d(x_i, x) > b} = 0.$$ 

For convenience, let $$f(x_i) =\sup_{x \in B(x_i, a)}\sum_{j=1}^m w_j^{S_r}(x)I_{d(x, x_j) > b}.$$ For $x_i \in \cup_{j \in J} B_j$, any $x \in B(x_i,a)$ trivially satisfies $d(x, B_i) \leq a$. Therefore, $f(x_i) = 0$ Since $|S_r^\X \cap \cup_{j \in J} B_j| \geq m(1-5\epsilon)$, and $f(x_i) \leq 1$ for all $1 \leq i \leq m$, we have that 
\begin{equation*}
\begin{split}
\frac{1}{m}\sum_1^m f(x_i) &= \frac{1}{m}(\sum_{x_i \in \cup_{i \in J} B_i} f(x_i) + \sum_{x_i \notin \cup_{i \in J}B_i} f(x_i))\\
&\leq  \frac{1}{m}(0 + 5\epsilon m (1)) \\
&= 5\epsilon.
\end{split}
\end{equation*}
Since all of our equations hold with probability $1-\delta$ over $S$ for sufficiently large $n$, this last one does as well. Since this entirely expression is always at most $1$ (regardless of $S$), and  since $\delta, \epsilon$ were arbitrary, we have that $$\lim_{n \to \infty} E_{S \sim \D^n}[\frac{1}{m}\sum_1^m f(x_i)] = 0,$$ which completes the proof.  
\end{proof}

\begin{proof}(\textbf{Corollary \ref{thm_kern_gen}})
For convenience, we let $W'$ represent the weight function described by $\ga(S,W, r)$. In particular, $W'_S$ and $W_{S_r}$ are the same classifier, where $S_r$ denotes the largest $r$-separated subset of $S$.

Relabel the points in $S$ so that $$S_r = \{(x_1, y_1), (x_2, y_2), \dots, (x_m, y_m)\},$$ with $ m \leq n$. We will also let $S_r^\X = \{x_1, x_2, \dots, x_m\}$. 

By Theorem \ref{thm_weight_general}, it suffices to show that for any $0 < a < b$, $$\lim_{n \to \infty}\mathbb{E}_{S \sim \D^n} [\frac{1}{m} \sum_{i=1}^{m}\sup_{x \in B(x_i,a)} \sum_{j=1}^{m} w_j^{S_r}(x)I_{d(x_i, x) > b}] = 0,$$ where $w_j$ are the weight functions corresponding to $W$. Fix $0 < a < b$, and let $\epsilon > 0$. 

Pick $\gamma > 0$ be such that $a+ 2\gamma < b$. Let $B_1, B_2, \dots, B_s$ be a $(\epsilon, \gamma, \alpha)$ decomposition of $\D$. By applying a Chernoff bound, for any $\delta > 0$ there exists $n_0$ such that for $n \geq n_0$, with probability $1-\delta$ over $S \sim \D^n$, $$|S_\X \cap \cup_1^s B_i| \geq (1-2\epsilon)n.$$  By Lemma \ref{lem_half}, $\frac{m}{n} \geq \frac{1}{2}$. It follows that $|S_r^\X \cap \cup_1^s B_i| \geq m(1-4\epsilon)$. 

Let $$J = \{i: |B_i \cap S_r^\X| \geq \frac{m\epsilon}{s}\}.$$ By Lemma \ref{lem_point_count}, $|S_r^\X \cap \cup_{i \in J} B_i| \geq m(1 - 4\epsilon) - m\epsilon = m(1 - 5\epsilon)$.

Next, consider any $x_i, x_j \in S_r^\X$, and let $x$ be a point such that $d(x_i, x) \leq a+2\gamma$ and $d(x_j, x) > b$. Recall that $W$ is constructed from kernel function $K$ and window parameter $h_n$. We then have that 
\begin{equation}\label{eqn_kern_ratio}
\begin{split}
\frac{w_j^S(x)}{w_i^S(x)} = \frac{K(\frac{d(x_j, x)}{h_n})}{K(\frac{d(x_i, x)}{h_n})}.
\end{split}
\end{equation}
Because $b > a + 2\gamma$, $\frac{d(x_j, x)}{d(x_i, x)} > 1$. Fix any $\beta > 0$. Because $\lim_{n \to \infty} h_n = 0$ and $\lim_{x \to \infty} \frac{K(cx)}{K(x)} = 0$ for $c > 1$, there exists $N$ such that for $n \geq N$, $$\frac{w_j^S(x)}{w_i^S(x)} \leq \frac{\beta\epsilon}{s}.$$ 

Fix $B_j$ with $j \in J$, and consider $x$ with $d(x, B_j) \leq a$. By the triangle inequality, $d(x, x') \leq a+ 2\gamma$ for all $x' \in B_j$. Then we have the following,
\begin{equation}\label{eqn_weight_bound}
\begin{split}
\sum_1^m w_i^{S_r}(x)I_{d(x_i, x) > b} & \stackrel{(a)}{=} \frac{\sum_1^m w_i^{S_r}(x)I_{d(x_i, x) > b}}{\sum_1^m w_i^{S_r}(x)} \\
&\stackrel{(b)}{\leq} \frac{\sum_1^m w_i^{S_r}(x)I_{d(x_i, x) > b}}{\sum_{x_i \in B_j} w_i^{S_r}(x)} \\
&\stackrel{(c)}{\leq} \frac{m\sup_{x_i: d(x_i, x)>b}w_i^{S_r}(x)}{m\epsilon/s \inf_{x_i \in B_j} w_i^{S_r}(x)} \\
&\stackrel{(d)}{\leq} \frac{\beta\epsilon/s}{\epsilon/s} = \beta.
\end{split}
\end{equation}

Equation $(a)$ holds because the total sum of weights is always 1, $(b)$ because all weights are nonnegative, $(c)$ because $|B_j \cap S_r^\X| \geq m\epsilon/s$, and $(d)$ because of equation \ref{eqn_kern_ratio}.

Let $$Z_i=\sup_{x \in B(x_i, a)}\sum_{j=1}^m w_j^{S_r}(x)I_{d(x, x_j) > b}.$$ For $x_i \in \cup_1^t B_j$, any $x \in B(x_i,a)$ trivially satisfies $d(x, B_i) \leq a$. By equation \ref{eqn_weight_bound}, it follows that $Z_i \leq \beta.$ Since $|\cup_{j \in J} B_j \cap S_r^\X| \geq m(1-5\epsilon)$ and $Z_i \leq 1$ for all $1 \leq i \leq m$, we have that 
\begin{equation*}
\begin{split}
\frac{1}{m}\sum_1^m Z_i &= \frac{1}{m}(\sum_{x_i \in \cup_{j \in J} B_j} Z_i+ \sum_{x_i \notin \cup_{j \in J} B_j} Z_i)\\
&\leq (1-5\epsilon)\beta + 5\epsilon.
\end{split}
\end{equation*}
Since all of our equations hold with probability $1-\delta$ over $S$ for sufficiently large $n$, this last one does as well. Since this entire expression is always at most $1$ (regardless of $S$), and  since $\delta, \epsilon, \beta$ were arbitrary, we have that $$\lim_{n \to \infty} E_{S \sim \D^n}[\frac{1}{m}\sum_1^m Z_i] = 0,$$ which completes the proof. 

\end{proof}

\section{Experimental Details}
\begin{figure}[ht]
\vskip 0.2in
\begin{center}
\subfloat[][Training Size = 20]{\includegraphics[width=.45\textwidth]{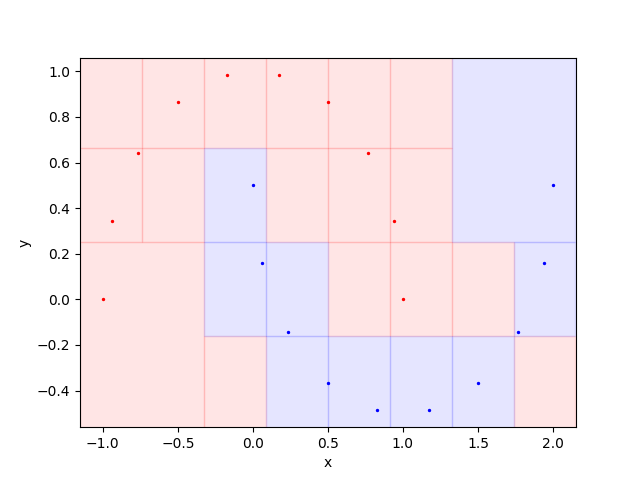}}\quad
   \subfloat[][Training Size = 50]{\includegraphics[width=.45\textwidth]{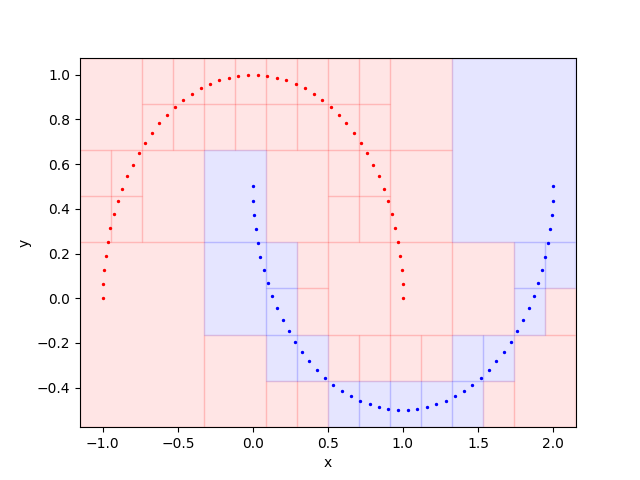}}\\
   \subfloat[][Training Size = 500]{\includegraphics[width=.45\textwidth]{visual500}}\quad
   \subfloat[][Training Size = 3000]{\includegraphics[width=.45\textwidth]{visual3000}}
\end{center}
\caption{A visualization of histograms learned with training data sampled from noiseless halfmoon. As training size grows, the histogram classifier becomes increasingly susceptible to adversarial examples in the blue regions.}
\label{fig:appendix_fig}
\vskip -0.2in
\end{figure}

\subsection{Optimal attacks against histogram classifiers}

Let $H$ be a histogram classifier, and let $(x,y)$ be any labeled example. Let $r > 0$ be some fixed robustness radius. Recall that an \textit{adversarial example} against $H$ at $(x,y)$ is any $x'$ such that $x' \in B(x,r)$ and $H(x') \neq y$. Note that if $H(x) \neq y$, then $x$ itself is an adversarial example. Conversely, if $H$ is astute at $(x,y)$ with radius $r$, then no adversarial example exists.

For arbitrary classifiers, finding adversarial examples at a given point can be challenging. However, recent work (Yang et. al. 2019) has shown that for non-parametric classifiers, there are tractable methods for doing so. The key insight is that non-parametric classifiers can be construed as a partitioning of input space into convex cells, with each cell having a given label. For example, Figure \ref{fig:appendix_fig} gives a visualization for these cells in a histogram classifier. 

Because these cells are convex, finding an adversarial example for $H$ at $(x,y)$ (here $x$ is a point in $\R^2$, and $y$ is a label) amounts to finding the closest cell $c \in H$ to $x$ such that $H(c) \neq y$. While Yang et. al. (Yang et. al. 2019) presents convex programming algorithms for doing this, the case of histograms in the $\ell_\infty$ metric is much simpler. 

As stated in definition 10, a histogram partitions the input space into hypercubes by iteratively splitting each cube into $2^d$ cubes with half the length. Therefore, the cells of a histogram are all hypercubes of varying sizes. For cell $c$, let $s(c)$ denote the length of the cube that $c$ corresponds to, and let $H(c)$ denote the label $H$ assigns to $c$. The key observation is that $c$ contains an adversarial example for $(x,y)$ if and only if $d(c, x) \leq s(c)/2 + r$, and $H(c) \neq y$. This yields the following algorithm:

Algorithm \ref{alg_hist_attack} was further optimized by utilizing nearest-neighbor type algorithms to find the ``closest" cells to $x$. This was done by grouping cells by their radii, and utilizing a separate nearest-neighbor data structure for all cells of a given radius. 

Although this algorithm doesn't have the same performance metrics as those presented in (Yang et. al. 2019), it was easily sufficient for computing the empirical astuteness for our experiments.

\begin{algorithm}[tb]
   \caption{Optimal attack algorithm for Histogram Classifiers}
   \label{alg_hist_attack}
\begin{algorithmic}
   \STATE {\bfseries Input:} Histogram $H$, labeled point $(x,y) \in \R^2 \times \{\pm 1\}$, robustness radius $r$
   \FOR{cell $c \in H$}
   \IF{$d(c,x) \leq s(c)/2 + r$ and $H(c) \neq y$}
   \STATE{Return $c$}
   \ENDIF
   \ENDFOR
\end{algorithmic}
\end{algorithm}

\end{document}